\documentclass[runningheads]{llncs}
\usepackage[T1]{fontenc}
\usepackage{graphicx}
\usepackage{booktabs}
\usepackage[misc]{ifsym}
\usepackage{algorithmic}

\usepackage{booktabs}  
\usepackage{amsmath} 
\usepackage{tikz}
\usepackage{pgfplots}
\usepackage{newfloat}
\usepackage{amsmath} 
\usepackage{listings}
\usepackage{tikz}
\usepackage{multirow}
\usetikzlibrary{calc}
\usepackage{graphicx}
\usepackage{pgfplots}  
 \usetikzlibrary{patterns}
\usepackage{pgfplots}
\usepackage{url} 
\usepackage{amssymb}

\usepackage{tikz}
\usepackage{pgfplots}

\usepackage{multirow}
\usetikzlibrary{calc}
\usepackage{graphicx}
\usepackage{pgfplots}  
 \usetikzlibrary{patterns}
\usepackage{pgfplots}

\usepackage{tikz}
\usepackage{pgfplots}
\pgfplotsset{compat=1.17}                 
\usetikzlibrary{pgfplots.groupplots} 
\usepackage{mwe}
\newcommand{\CT}{\mathbf{CT}}
\newcommand{\Reff}{\mathbf{Reff}}

\begin{document}

\title{Ramanujan Graph Rewiring with Non Negative Resistance Curvature}

\titlerunning{Ramanujan Graph Rewiring}

\author{Author information scrubbed for double-blind reviewing}
\author{Hugo Attali  \and
Rachid El Jouhri}

\authorrunning{H. Attali and R. El Jouhri }

\institute{Université Sorbonne Paris Nord, CNRS, LIPN, France \email{\{attali\}@lipn.univ-paris13.fr}}

\maketitle              

\begin{abstract}

Graph Neural Networks (GNNs) have emerged as a powerful paradigm for learning on graph‑structured data by iteratively propagating and aggregating information across edges. However, conventional message passing schemes often suffer from over-squashing, whereby exponentially large neighborhoods are compressed into fixed‑dimensional embeddings, impeding effective long‑range dependency learning. In this work, we introduce Ramanujan Propagation, a  graph rewiring strategy that leverages Ramanujan graphs to alleviate topological bottlenecks in GNNs. We first establish that suitably chosen Ramanujan graphs guarantee non-negative resistance curvature, which mitigates over‑squashing and facilitates efficient information flow. We then propose an algorithmic framework to construct a Ramanujan rewired graph that preserves the local connectivity of the original graph. Our experiments demonstrate that our method outperforms nine state‑of‑the‑art rewiring techniques. These results establish Ramanujan graphs as a rigorous structural prior for scalable, topology‑aware message passing in GNNs.

\end{abstract}

\section{Introduction}

Graph Neural Networks (GNNs) \cite{goller1996learning,gori2005new,scarselli2008graph,bruna2013spectral} have become a powerful framework for learning on graph structured data, tackling tasks at the node, edge, or graph level. These architectures find applications in diverse fields
such as chemistry, information retrieval, social networks, knowledge graphs
\cite{zhou2020graph,wu2020comprehensive}.

GNNs iteratively update each node’s representation by exchanging messages along edges, thereby enriching the quality of the learned embeddings \cite{gilmer2017neural}. This local message passing paradigm excels in homophilic graphs, where adjacent nodes tend to share the same labels. However, many real world problems exhibit long range dependencies that cannot be captured through purely local interactions \cite{platonov2023critical,bamberger2025measuring}. In such cases, Message Passing Neural Networks (MPNNs) often struggle to learn high‑quality representations due to the over-squashing phenomenon : as the number of layers increases, each node must compress an exponentially growing set of incoming messages into a fixed‑size vector, causing information from distant nodes to be effectively lost \cite{alon2020bottleneck,UNDERSTANDING_bottlenecks,di2023does,di2023over}. Moreover, the underlying graph topology critically influences long range message propagation. Irregular subgraphs may lead to structural overfitting \cite{bechler-speicher2024graph}, while bottleneck regions and sparse connectivity exacerbate over-squashing \cite{alon2020bottleneck,UNDERSTANDING_bottlenecks}. Conversely, overly dense subgraphs can induce oversmoothing \cite{rusch2023survey}, where node representations become indistinguishable \cite{BORF}.
To address these limitations, some work has proposed various multi‑hop and global propagation schemes that go beyond one‑hop message passing \cite{mixhop,GLOGNN,song2023ordered,chen2023node,gutteridge2023drew}. Instead of only scaling up MPNN capacity, recent work studies graph rewiring \cite{attali2026graph}: modifying the topology to ease information flow and to mitigate over-squashing and oversmoothing. These methods analyze how structure shapes message passing and then add or delete edges to shorten effective path lengths. Prior studies target topological bottlenecks and alter connectivity using curvature-based measures \cite{Forman_curvature,Olivier_curvature,UNDERSTANDING_bottlenecks}, effective resistance heuristics \cite{chandra1989electrical}, and spectral criteria linked to Cheeger constant \cite{alon1984eigenvalues,cheeger2015lower,chung1997spectral}. In practice, rewiring can be a one-shot preprocessing step or coupled to training, trading structural edits to improve long-range propagation and mitigate over-squashing.

\paragraph{\textbf{Main contributions.}}

This paper introduces Ramanujan with Non-negative Resistance-curvature Propagation, a graph-rewiring strategy that employs \(d\)-regular Ramanujan graphs to alleviate the over-squashing effect in GNNs. We prove that, under the assumption that \(d\) scales sufficiently with the graph size, these graphs guarantee non-negative resistance curvature. We further show that this class of graphs reduces over-squashing and that, under this condition, Ramanujan graphs possess highly desirable structural properties for message passing, including favorable spectral characteristics, low diameter, and low effective resistance, thereby enabling more efficient message exchange across the graph. Extensive experiments demonstrate that Ramanujan Propagation consistently outperforms nine state-of-the-art rewiring techniques.

\paragraph{\textbf{Reproducibility.}} Our code to reproduce the experiments in this paper is available at the following anonymous repository:\footnote{\url{https://github.com/Hugo-Attali/ECML-PKDD_2026_Ramanujan_Graph_Rewiring}}.

\section{Notations}

We denote a graph by $G=(\mathcal{V},\mathcal{E})$ with $N=|\mathcal{V}|$. Its adjacency matrix $\mathbf{A}\in\{0,1\}^{N\times N}$ has entries
$\mathbf{A}_{ij}=1$ iff $(i,j)\in\mathcal{E}$, else $0$. The degree of node $i$ is $d_i=\sum_j \mathbf{A}_{ij}$ and $\mathbf{D}=\mathrm{diag}(d_1,\dots,d_N)$. The normalized Laplacian is
$\mathbf{L}_{\mathrm{norm}}=\mathbf{I}-\mathbf{D}^{-1/2}\mathbf{A}\mathbf{D}^{-1/2}$,
with spectrum $0=\lambda_1\le\cdots\le\lambda_N\le 2$. The diameter is $\mathrm{diam}(G)=\max_{i,j}\mathrm{dist}_G(i,j)$,
where $\mathrm{dist}_G(i,j)$ is the length of a shortest path. Finally, we denote by $h_i^{(l)}\in\mathbb{R}^d$ the embedding of node $i$ at layer $l$ of a GNN.

\section{Related work}

\subsection{Limitations of Message Passing}
\label{sec:limitations}

Message passing in Graph Neural Networks (GNNs) alternates between an \(\operatorname{AGGREGATE}\) operation, whereby each node \(i\) collects and compresses information from its one hop neighborhood  $\mathcal{N}$, and an \(\operatorname{UPDATE}\) operation, which refines the node’s embedding by combining its previous state with the aggregated message:
\begin{equation}
\begin{aligned}
    m_{i}^{(\ell)} &= \operatorname{AGGREGATE}^{(\ell)}\bigl(h_i^{(\ell-1)},\,\{h_j^{(\ell-1)}:j\in\mathcal{N}(i)\}\bigr),\\
    h_{i}^{(\ell)} &= \operatorname{UPDATE}^{(\ell)}\bigl(h_i^{(\ell-1)},\,m_{i}^{(\ell)}\bigr).
\end{aligned}
\label{eq:message_passing}
\end{equation}

By definition, propagating information between nodes at graph distance \(r\) requires at least \(r\) layers.  However, increasing depth gives rise to two inherent phenomena that harm effective long range dependency.  

\paragraph{\textbf{Over-squashing.}} Over-squashing occurs when exponentially large \(R\)-hop neighborhoods are compressed into fixed‐dimensional messages, leading to loss of information from distant nodes \cite{alon2020bottleneck}. Following \cite{arnaiz2025oversmoothing}, the overloaded term over-squashing is best understood by separating two distinct phenomena. 

First, computational bottlenecks arise from the intrinsically local nature of message passing in GNNs: the influence of perturbations at node $j$ on node $i$ typically decays exponentially with their graph distance $d_G(i,j)$ \cite{UNDERSTANDING_bottlenecks,di2023over,di2023does}. 

Second, topological bottlenecks are constraints imposed by the graph itself, sparse bridges, large diameters, and bottleneck substructures that suppress information flow regardless of network depth or architecture \cite{alon2020bottleneck}. 
Keeping these notions separate avoids conflating architecture-induced limitations with structure-induced constraints and clarifies how to target each with appropriate remedies.


\paragraph{\textbf{Oversmoothing.}} 
\label{oversmooth}
Oversmoothing refers to the phenomenon 
whereby iterative neighborhood aggregation progressively aligns node representations, reducing their variance and compromising the model’s ability to distinguish nodes with different structural roles \cite{oversmoothing,rusch2023survey}. Oversmoothing is typically assessed via the Dirichlet energy of the layer representations:

\begin{equation}
\mathcal{E}_{\text{node}}(\mathbf{H})
= \frac{1}{2n}\sum_{i=1}^{n}\;\sum_{j\in\mathcal{N}(i)} \big\lVert h_i - h_j \big\rVert_2^{2}.
\label{diri}
\end{equation}

When the Dirichlet energy decreases, neighboring nodes become more similar, and it tends to vanish as oversmoothing intensifies.

This issue is particularly pronounced in dense substructures, such as large cliques, which exacerbate the effect by diffusing an excessive amount of information \cite{BORF}. These topological constraints demonstrate that simply stacking additional layers fails to overcome the topological bottleneck induced by over-squashing or to preserve embedding diversity under oversmoothing.

\subsection{Graph Rewiring}

Rewiring transforms a graph \(G=(\mathcal{V},\mathcal{E})\) into a new graph \(G^*=(\mathcal{V}^*,\mathcal{E}^*)\), where the modified edge set \(\mathcal{E}^*\) is tailored to enhance information propagation to mitigate over-squashing.

\noindent Early approaches constrained rewiring to the original vertex set \(\mathcal{V}^* = \mathcal{V}\), focusing solely on edge modifications.

Local approaches identify structural bottlenecks by focusing on edges exhibiting strongly negative curvature. In graph-theoretic terms, curvature quantifies the deviation of the local connectivity pattern around a node or edge from a “flat” baseline (e.g., Euclidean-like grids or tree-like structures), thereby exposing geometric irregularities. Curvature-based rewiring accordingly performs targeted edge modifications such as flips or addition guided by established curvature indices, notably Ollivier–Ricci and Forman curvature \cite{UNDERSTANDING_bottlenecks,giraldo2023trade,BORF,fesser2023mitigating}, to mitigate negatively curved edges and alleviate the associated bottlenecks.

By contrast, global strategies reshape the topology via spectral objectives, for example, maximizing the normalized spectral gap \cite{banerjee2022oversquashing,FOSR} or minimizing the effective resistance \cite{chandra1989electrical}, which aggregates all parallel communication pathways between node pairs \cite{resit}. Intuitively, a larger spectral gap promotes expansion and faster mixing, mitigating bottleneck substructure that throttle information flow, while lower effective resistance indicates many short, parallel routes between nodes, reducing congestion and facilitating long-range communication.

More recently, another line of work incorporates node features into the rewiring process. For example, \cite{attali2024delaunay} reconstructs the entire graph via a Delaunay triangulation in node-feature space. The authors show that applying triangulation helps avoid edges exhibiting strongly negative or strongly positive curvature. \cite{attali2025dynamic} extends this principle  dynamically by learning to select relevant triangles from multiple 
graph views, jointly optimizing triangle selection and downstream performance. Another approach \cite{linkerhagner2025joint}, jointly updates the graph structure and the node features to maximize spectral alignment between the node features and the graph structure, thereby promoting coherence between topology and attributes.

Some other works propose expander based algorithms leverage algebraic group constructions to build graphs with provably large spectral gaps without increasing overall density \cite{deac2022expander,wilson2024cayley}. 
Finally, PANDA \cite{choi2024panda} takes a different approach by adapting node feature dimensionalities so that central nodes retain richer representations, thereby alleviating information compression without altering the original topology .

\subsection{Graphs with Nonnegative Resistance Curvature}

Recently some works  \cite{devriendt2024graphs,devriendt2024graph} introduce the notion of resistance curvature on graphs derived from the probabilities that an edge belongs to a weighted spanning tree and define the classes of graphs with Non-negative Resistance curvature ($RN$) or Positive Resistance curvature ($RP$) as follows.

\noindent For a finite, simple and unweighted graph the  resistance curvature of a node is defined as :

\begin{equation}
  p(v)
  = 1 - \frac{1}{2} \sum_{u \in \mathcal{N}(v)} \mathbf{R_{eff}}(v,u) \quad with \qquad
p(v)\;<\;1
\end{equation}

where $\mathbf{R_{eff}}$ denotes the effective resistance \cite{chandra1989electrical} between nodes \(v\) and \(u\).

A graph is $RN$ if $ p(v)\;\ge\;K \ge 0$
for every node.
The authors prove that $RN$ graphs exhibit strong spectral properties. First the spectral gap of the Laplacian satisfies : 
\begin{equation}
\lambda_2(G)\;\ge\;2K,
\end{equation}

Secondly, the diameter of any RN graph can be tightly controlled by
\begin{equation}
      \mathrm{diam}(G)\;\le\;\biggl\lceil \sqrt{\frac{\Delta}{K}}\;\log\bigl|V(G)\bigr|\biggr\rceil,
\label{diammm}
\end{equation}
where $\Delta$ denotes the maximum degree of $G$.

These novel graph classes offer a powerful indicator of long‑range connectivity, moving beyond traditional, purely edge‐centric metrics.

\section{Understanding over-squashing through the Lens of Node-Level Resistance Curvature in GNNs}

We connect the node-level \emph{resistance curvature} introduced by ~\cite{devriendt2024graph} to a classical notion of random-walk communication: the \emph{commute time}. Given two nodes $u,v\in V$, one can quantify how difficult it is to reach $v$ from $u$ by the average number of steps needed until the walk visits $v$ for the first time.
We define the commute time between $u$ and $v$, denoted $\mathbf{CT}(u,v)$, as the expected time for a random walk to travel from $u$ to $v$ and then return to $u$.

Large commute time $\mathbf{CT}(u,v)$ indicates that $v$ is difficult to access from $u$ by local diffusion, and this limited accessibility translates into weaker pairwise influence in message passing.
In particular, \cite{di2023does} quantify over-squashing by the collapse of Jacobian sensitivities $\bigl\|\partial h^{(l)}_v/\partial h^{(0)}_u\bigr\|$ and show that larger commute times are associated with a faster decay of these cross-node sensitivities, i.e., more severe over-squashing.

\begin{theorem}[Commute time bound via resistance curvature]
Let $G$ be a connected, undirected graph.
For any two nodes $u,v\in V$, let $u=x_0\sim x_1\sim\cdots\sim x_k=v$ be a shortest path (so $k=d_G(u,v)$).
Define $p_\star := \min_{0\le i\le k} p(x_i),$ the minimum (along this path) of the node-level resistance curvature $p(\cdot)$.
Then the commute time satisfies
\[
\mathbf{CT}(u,v)\;\le\;4\,|E|\,k\,\bigl(1-p_\star\bigr).
\]

\end{theorem}

\medskip

A larger $p_\star$ tightens the bound and forces a smaller commute time $\mathbf{CT}(u,v)$, indicating that $u$ and $v$ are easily connected by local diffusion.
Because message passing propagates information through such local diffusive mixing, small $\mathbf{CT}(u,v)$ is consistent with stronger long-range influence.
In particular, in the sense of \cite{di2023does}, it rules out the high-commute-time regime where pairwise Jacobian sensitivities $\bigl\|\partial h^{(l)}_v/\partial h^{(0)}_u\bigr\|$ collapse, supporting the interpretation that increasing $p_\star$ mitigates  over-squashing.

\noindent The proof is provided in Appendix \ref{CTCURV} .

\section{Ramanujan Propagation}

In this section, we review the key spectral and topological properties of \(d\)-regular Ramanujan graphs and show how they motivate our rewiring algorithm. 

\paragraph{\textbf{Definition 1}}

A finite connected $d$-regular graph $G$ is called a \emph{Ramanujan graph} \cite{lubotzky1988ramanujan}  if every eigenvalue $\mu$ of its adjacency matrix $A$  satisfies :
\begin{equation}
    |\mu|\le2\sqrt{d-1}.
\end{equation}

 Although Ramanujan graphs demonstrate strong spectral properties, one can further enhance their structural advantages by imposing non‑negative resistance curvature \cite{devriendt2024graphs}, thereby improving the topology for efficient message passing. 

\subsection{Ramanujan Graphs with Nonnegative Resistance Curvature}

In order to situate our approach within the context of graph rewiring methods, we first note that any graph designed to facilitate long‑range information flow must satisfy a number of desirable structural properties. Specifically, such a graph should exhibit strong spectral expansion, a small diameter, low effective resistance between node pairs (to ensure efficient communication), limited variance in node degree, and built‑in robustness against the over-squashing of messages.
We theoretically prove that, under a suitable choice of degree $d$, a Ramanujan graph satisfies all of these conditions, thereby offering a principled structural prior for long-range propagation.

\paragraph{\textbf{Mitigate over-squashing}}

We prove that with a suitable choice of degree  $d$ we can ensure that the Ramanujan graph is RN. According to Theorem 1 this kind of graph mitigates over-squashing.
    
\begin{theorem}
\label{valued}
Let \(G\) be a Ramanujan graph on \(N\) vertices of degree \(d\),
and let \(C>0\) be the absolute constant from~\cite{glock2024hamilton}.
If
\[
d \;\geq\; 4\,C^{2}\,(\log N)^{2/3},
\]
then \(G\) has non-negative resistance curvature.

\end{theorem}

\noindent The proof is provided in Appendix \ref{RNGRAPH} .

\begin{remark}
\label{remark1_}
While Theorem~\ref{valued} requires $d \geq 4C^2(\log N)^{2/3}$, 
in practice the smaller degree $d =\lceil 4\,(\log N)^{2/3} \rceil$ 
already yields non-negative resistance curvature on all our datasets, 
i.e.\ $p(v) \geq 0$ for every node $v$. We adopt this value in what follows.
\end{remark}

By constructing the Ramanujan graph in this way, we ensure it is $RN$, thereby further improving the sensitivity between any two nodes.

\paragraph{\textbf{Low Diameter}}

Using Theorem \ref{valued} we can show that the Ramanujan $RN$ graph admits low diameter.

\begin{corollary}
\label{diammm}
Let \(G\) be a finite, connected Ramanujan graph on \(N\) vertices of degree \(d\) with Non-negative resistance curvature.  Then

\[
\operatorname{diam}(G) \le 1 + \frac{3\,\log(N)}{\log (\log N) }.
\]

\noindent The proof is provided in Appendix \ref{diamgraph} .

\end{corollary}

By reducing its diameter to the sub-logarithmic scale 
$\mathcal{O}\!\left(\frac{\log N}{\log \log N}\right)$, the Ramanujan \(RN\) graph brings distant vertices within only a few hops significantly enhancing long‑range connectivity, shortening path lengths, and thus substantially improving the efficiency of information propagation.


\paragraph{\textbf{Low Resistance Edges}}

According to \cite{resit}, effective resistance captures structural bottlenecks more accurately than curvature by taking into account a more global graph structural point of view. Edges with low effective resistance thus facilitate information flow between neighboring nodes. We prove that, in a Ramanujan \(RN\) graph, the maximum effective resistance is tightly controlled by the graph’s degree \(d\), and that the resulting uniformly low-resistance edges substantially improve communication between distant vertices.

\begin{figure*}[t]
\centering
\begin{tikzpicture}[scale=0.999]

\begin{groupplot}[
    group style={
        group size=3 by 1,
        horizontal sep=1cm
    },
    width=0.38\textwidth,
    height=5.2cm,
    xlabel={Top-$f$ fraction},
    ylabel={Mean resistance},
    xtick={0,0.2,0.4,0.6,0.8,1.0},
    ytick style={draw=none},
    tick style={line width=0.3pt},
    axis line style={line width=0.3pt},
    tick label style={font=\footnotesize},
    label style={font=\small},
    title style={font=\small, yshift=-1ex},
    grid=major,
    grid style={dotted, gray!30}
]

\nextgroupplot[
    title={ENZYMES},
    ymin=0.2, ymax=0.75
]
\addplot[blue!80!black, thick, mark=*, mark size=1pt, smooth] table[row sep=\\] {
0.10 0.69 \\ 0.20 0.63 \\ 0.30 0.60 \\ 0.40 0.58 \\ 0.50 0.56 \\ 0.60 0.545 \\ 0.70 0.53 \\ 0.80 0.517 \\ 0.90 0.503 \\ 1.00 0.49 \\
};
\addplot[orange!80!black, thick, mark=x, mark size=1.5pt, smooth] table[row sep=\\] {
0.10 0.25 \\ 0.20 0.247 \\ 0.30 0.244 \\ 0.40 0.241 \\ 0.50 0.238 \\ 0.60 0.236 \\ 0.70 0.234 \\ 0.80 0.233 \\ 0.90 0.231 \\ 1.00 0.230 \\
};
\addplot[green!50!black, thick, mark=triangle*, mark size=1.5pt, smooth] table[row sep=\\] {
0.10 0.486 \\ 0.20 0.486 \\ 0.30 0.486 \\ 0.40 0.486 \\ 0.50 0.486 \\ 0.60 0.486 \\ 0.70 0.486 \\ 0.80 0.486 \\ 0.90 0.486 \\ 1.00 0.486 \\
};

\nextgroupplot[
    title={PROTEINS},
    ylabel={},
    ymin=0.23, ymax=0.75
]
\addplot[blue!80!black, thick, mark=*, mark size=1pt, smooth] table[row sep=\\] {
0.10 0.715 \\ 0.20 0.645 \\ 0.30 0.615 \\ 0.40 0.59 \\ 0.50 0.57 \\ 0.60 0.555 \\ 0.70 0.54 \\ 0.80 0.53 \\ 0.90 0.518 \\ 1.00 0.50 \\
};
\addplot[orange!80!black, thick, mark=x, mark size=1.5pt, smooth] table[row sep=\\] {
0.10 0.26 \\ 0.20 0.257 \\ 0.30 0.254 \\ 0.40 0.252 \\ 0.50 0.25 \\ 0.60 0.248 \\ 0.70 0.246 \\ 0.80 0.245 \\ 0.90 0.244 \\ 1.00 0.243 \\
};
\addplot[green!50!black, thick, mark=triangle*, mark size=1.5pt, smooth] table[row sep=\\] {
0.10 0.50 \\ 0.20 0.50 \\ 0.30 0.50 \\ 0.40 0.50 \\ 0.50 0.50 \\ 0.60 0.50 \\ 0.70 0.50 \\ 0.80 0.50 \\ 0.90 0.50 \\ 1.00 0.50 \\
};

\nextgroupplot[
    title={REDDIT-BINARY},
    ylabel={},
    ymin=0.17, ymax=1.05
]
\addplot[blue!80!black, thick, mark=*, mark size=1pt, smooth] table[row sep=\\] {
0.10 1.00 \\ 0.20 1.00 \\ 0.30 1.00 \\ 0.40 1.00 \\ 0.50 0.995 \\ 0.60 0.985 \\ 0.70 0.965 \\ 0.80 0.935 \\ 0.90 0.90 \\ 1.00 0.86 \\
};
\addplot[orange!80!black, thick, mark=x, mark size=1.5pt, smooth] table[row sep=\\] {
0.10 0.19 \\ 0.20 0.188 \\ 0.30 0.187 \\ 0.40 0.186 \\ 0.50 0.185 \\ 0.60 0.184 \\ 0.70 0.184 \\ 0.80 0.183 \\ 0.90 0.183 \\ 1.00 0.183 \\
};
\addplot[green!50!black, thick, mark=triangle*, mark size=1.5pt, smooth] table[row sep=\\] {
0.10 0.50 \\ 0.20 0.50 \\ 0.30 0.50 \\ 0.40 0.50 \\ 0.50 0.50 \\ 0.60 0.50 \\ 0.70 0.50 \\ 0.80 0.50 \\ 0.90 0.50 \\ 1.00 0.50 \\
};

\end{groupplot}

\path (group c1r1.south west) -- node[below=1.5cm, draw=none] {
    \begin{minipage}{0.8\textwidth}
    \centering
    \begin{tikzpicture}
    \begin{axis}[
        hide axis,
        xmin=0, xmax=1,
        ymin=0, ymax=1,
        legend columns=3,
        legend style={draw=none, font=\footnotesize, column sep=1em}
    ]
    \addlegendimage{blue!80!black, thick, mark=*}
    \addlegendentry{Original}
    \addlegendimage{green!50!black, thick, mark=triangle*}
    \addlegendentry{Cayley graph}
    \addlegendimage{orange!80!black, thick, mark=x}
    \addlegendentry{Ramanujan $RN$ graph}
    \end{axis}
    \end{tikzpicture}
    \end{minipage}
} (group c3r1.south east);

\end{tikzpicture}
\caption{Mean resistance among top high-resistance edges across three datasets on 50 graphs. Ramanujan $RN$ graph achieves the lowest resistance, indicating a better global connectivity.}
\end{figure*}
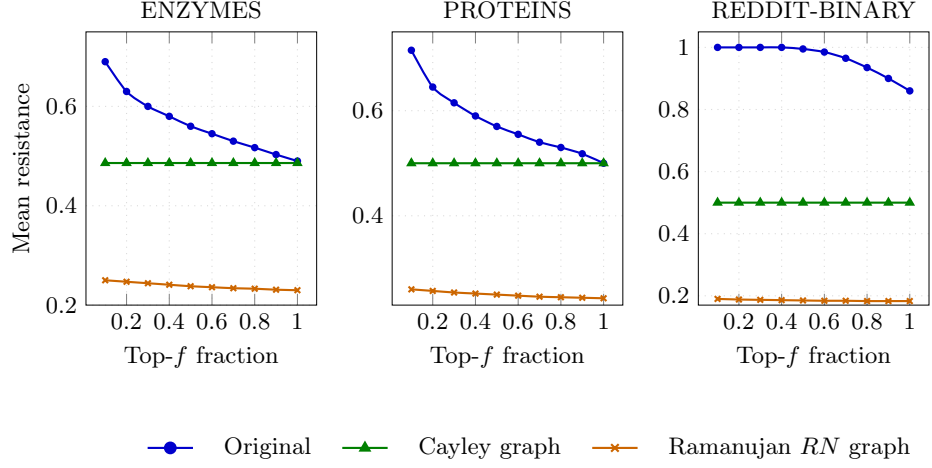

\begin{corollary}
Let $G$ be a Ramanujan \(RN\) graph of degree $d$ on $N$ vertices . Then the maximum effective resistance $R_{\max}$ satisfies :
\[
 R_{\max} \;\le \frac{2}{\; d - 2\sqrt{d-1}} ,
\]

\noindent The proof is provided in Appendix \ref{RESGRAPH}.

\end{corollary}



In Figure 1, we plot the mean effective resistance on 50 graphs as a function of the top-$f$ fraction of edges for three datasets from TUDataset \cite{morris2020tudataset} (Enzymes, Proteins, and Reddit-Binary). As $f$ decreases and we focus on the edges with the highest resistance, i.e.\ those most prone to over‑squashing \cite{resit}, both our Ramanujan $RN$  and the Cayley graph (CGP) \cite{wilson2024cayley}  consistently achieve lower mean resistance than the original graphs. Moreover, the Ramanujan rewired graphs attain the lowest resistance overall, outperforming CGP on every dataset indicating improved propagation of information between distant nodes. 
Finally, the resistance curves for both Cayley and Ramanujan graphs are nearly flat, indicating a highly concentrated distribution of edge effective resistances, which is consistent with their regular degree.

\paragraph{\textbf{Mitigating Structural Overfitting}}

Ramanujan graphs are \(d\)-regular, so every node aggregates exactly \(d\) messages at each message passing step, thereby mitigating the aggregation bottleneck that affects graphs with highly variable degrees \cite{arnaiz2025oversmoothing}. In practice \cite{bechler-speicher2024graph} show that, GNNs tend to overfit the provided graph structure even when it carries no predictive signal. By contrast, the constant degree in a \(d\)-regular Ramanujan graph prevents this overfitting to local topological irregularities, resulting in stronger generalization and more stable, faster convergence under gradient‑descent \cite{bechler-speicher2024graph}.

Thus, a graph’s ability to propagate information effectively is determined by a balance between graph structural propagation potential and geometric coherence. Consequently, to perform efficiently in graph learning tasks with long-range dependencies, the value of \(d\) must be chosen carefully. Theorem \ref{valued} provides a size-dependent criterion based solely on the number of vertices \(N\), guaranteeing non-negative resistance curvature without any heuristic parameter tuning.



\subsection{Ramanujan Graph Rewiring with Non-Negative Resistance Curvature Algorithm}
\label{algogo}
Given an input graph \(G=(V,E)\) with \(N=|V|\), we set $d = \lceil 4(\log N)^{2/3} \rceil$, following Remark~\ref{remark1_}, so that the RN condition is satisfied. To preserve the local geometry of the input graph, we compute a positional representation \(z_v \in \mathbb{R}^p\) for each node \(v\in V\) using Laplacian or random-walk positional encodings~\cite{dwivedi2021graph,dwivedi2022long}, and define the pairwise distance matrix \(\Delta_{uv}=\|z_u-z_v\|_2^2\).

We then construct \(G^\star=(V,E^\star)\) over the same vertex set \(V\) using Friedman's permutation-cycle construction~\cite{friedman2003proof}, which provides a standard random \(d\)-regular graph model commonly used in the study of Ramanujan graphs. At each step, we sample \(P\) candidate permutation cycles and select the one whose induced edges connect the closest nodes in the positional space, as measured by \(\Delta_{uv}\). This preserves \(d\)-regularity by construction while encouraging edges between nodes that are already nearby in the original graph.

Finally, the preprocessing cost is dominated by the computation of \(\Delta\), and is substantially lower than that of curvature or resistance based rewiring methods. We detail this complexity in Section~\ref{Time} and provide ablations in Section~\ref{ablation}, showing that preserving positional proximity is crucial for downstream performance.

\section{Experiments}

\label{XXXP}
We empirically evaluate the effectiveness of our Ramanujan $RN$ graph based rewiring  on a diverse set of graph classification benchmarks. First, we consider six datasets from the TUDataset repository \cite{morris2020tudataset}  which is widely used to assess rewiring methods designed to mitigate over-squashing, since their graph structures in relation to the downstream task require the propagation of long‑range dependencies \cite{FOSR}. We also evaluate our approach on benchmarks that require extensive receptive fields and long range dependencies on two Long‑Range Graph Benchmark datasets : regression on Peptides‑Struct and  multi-label classification on Peptides‑Func \cite{dwivedi2022long}.
Since the degree induced by RNRP depends on graph size (Theorem~\ref{valued}), we report the dataset statistics in the supplementary material.

\subsection{Evaluation Protocol}

\paragraph{\textbf{Baselines}} We compare our method to nine state‑of‑the‑art rewiring techniques: the Fully‑Adjacent (FA) strategy, which appends a full adjacency matrix at the last layer \cite{alon2020bottleneck}; the diffusion‑based method DIGL\cite{DIGL} ; the curvature‑based methods SDRF\cite{UNDERSTANDING_bottlenecks}  and BORF \cite{BORF} ; the spectral rewiring method FoSR\cite{FOSR}  ; the resistance‑based approach GTR \cite{resit} ; LASER\cite{barbero2024localityaware} a Random Walk Rewiring method ;
PANDA,  that dynamically adjusts the model’s width based on the node structure \cite{choi2024panda} ; and two expander‑graph‑based rewiring techniques \cite{deac2022expander,wilson2024cayley}.

\paragraph{\textbf{TUDataset configuration}} To ensure a fair comparison, we follow exactly the experimental setup from \cite{wilson2024cayley}: training all models on an 80\%/10\%/10\% train/validation/test split, with early stopping after 100 epochs of no improvement in validation loss. We evaluate each approach using two common GNN architectures a four‑layer Graph Convolutional Network (GCN) \cite{gcn} and a four‑layer Graph Isomorphism Network (GIN) \cite{xu2018powerful}, each configured with 64 hidden units per layer, 0.5 dropout, and batch normalization after every layer. Optimization and learning rate schedules follow the original implementations \cite{wilson2024cayley}. Results in Table 1 report averages over 20 random seeds with 95\% confidence intervals, and all baseline scores are taken from \cite{wilson2024cayley}.

\paragraph{\textbf{LRGB configuration}}
For Peptides-struct and Peptides-func \cite{dwivedi2022long},
we replicate the experimental settings from \cite{barbero2024localityaware,choi2025fractalinspired}, we use  a 70\% /15\% /15\% train/validation/test split with a fixed 5-layer GCN Backbone \cite{gcn} and we choose the hidden dimension in order to respect the 500k parameter budget. 
The reported results table \ref{res_tab2} are averaged across 5 random seeds.
Some baselines are taken from \cite{choi2025fractalinspired} and we add the Delaunay Rewiring method \cite{attali2024delaunay}.

\paragraph{\textbf{Ramanujan RN Graph Configuration}}

Unlike most rewiring techniques, the Ramanujan \(RN\) approach introduces no extra hyperparameters for the rewiring preprocessing  (e.g., the number of edges to add or to delete). In our experiments, we apply  Ramanujan \(RN\) graph rewiring either at every layer or solely on the final layer, depending on the dataset. To preserve proximity, we embed nodes with an eight‑dimensional Laplacian encoding across all datasets \cite{dwivedi2022long}.

\subsection{Results}
\begin{table}[t]
  \centering
  \scriptsize
  \setlength{\tabcolsep}{3pt}
  \renewcommand{\arraystretch}{1.0}
  \resizebox{\columnwidth}{!}{%
  \begin{tabular}{lcccccc}
    \toprule
    \textbf{Model} & \textbf{REDDIT-B} & \textbf{IMDB-B} & \textbf{MUTAG} & \textbf{ENZYMES} & \textbf{PROTEINS} & \textbf{COLLAB} \\
    \midrule
    GCN       & $77.74_{\pm 1.59}$         & $60.50_{\pm 2.73}$     & $74.75_{\pm 4.03}$     & $29.08_{\pm 2.36}$     & $66.65_{\pm 1.93}$     & $70.49_{\pm 1.63}$     \\
    $+$ FA    & OOM                        & $48.95_{\pm 1.65}$     & $70.25_{\pm 4.61}$     & $28.67_{\pm 3.69}$     & $71.07_{\pm 1.51}$     & $72.04_{\pm 0.77}$     \\
    $+$ DIGL  & $77.35_{\pm 1.21}$         & $49.60_{\pm 2.44}$     & $70.50_{\pm 5.05}$     & $30.83_{\pm 1.54}$     & $72.72_{\pm 1.42}$     & $56.47_{\pm 0.87}$     \\
    $+$ SDRF  & $77.98_{\pm 1.48}$         & $59.00_{\pm 2.25}$     & $74.00_{\pm 3.46}$     & $26.67_{\pm 2.00}$     & $67.28_{\pm 2.17}$     & $71.33_{\pm 0.81}$     \\
    $+$ FoSR  & $77.75_{\pm 1.39}$         & $59.75_{\pm 2.36}$     & $75.25_{\pm 5.72}$     & $24.17_{\pm 3.01}$     & $70.85_{\pm 1.62}$     & $67.22_{\pm 1.37}$     \\
    $+$ BORF  & OOT                        & $48.90_{\pm 0.90}$     & $76.75_{\pm 0.04}$     & $27.83_{\pm 0.03}$     & $67.41_{\pm 0.02}$     & OOT                      \\
    $+$ GTR   & $79.03_{\pm 1.25}$         & \underline{$60.70_{\pm 2.08}$}     & $76.50_{\pm 4.19}$     & $25.33_{\pm 2.93}$     & $72.99_{\pm 1.96}$     & \underline{$72.60_{\pm 1.03}$}     \\
    $+$ PANDA & \underline{$87.28_{\pm 1.03}$} & $\mathbf{68.35_{\pm 2.35}}$ & $76.75_{\pm 5.53}$ & $30.67_{\pm 2.02}$     & $70.13_{\pm 1.52}$     & $\mathbf{73.85_{\pm 0.70}}$ \\
    $+$ EGP   & $67.55_{\pm 1.20}$         & $59.70_{\pm 2.37}$     & $70.50_{\pm 4.74}$     & $27.58_{\pm 3.26}$     & \underline{$73.30_{\pm 2.52}$} & $69.47_{\pm 0.97}$ \\
    $+$ CGP   & $67.05_{\pm 1.48}$         & $56.20_{\pm 1.83}$     & $\mathbf{83.75_{\pm 3.60}}$ & \underline{$31.00_{\pm 2.40}$} & $73.04_{\pm 1.29}$ & $69.63_{\pm 0.73}$ \\
    \midrule
    $+$ RNRP  & $\mathbf{90.33_{\pm 2.26}}$ & $54.50_{\pm 2.12}$ & \underline{$79.25_{\pm 2.95}$} & $\mathbf{43.33_{\pm 2.60}}$ & $\mathbf{74.03_{\pm 2.07}}$ & $68.60_{\pm 1.61}$ \\
    \midrule
    GIN       & $84.60_{\pm 1.45}$         & $71.25_{\pm 1.51}$     & $80.50_{\pm 5.14}$     & $35.67_{\pm 2.80}$     & $70.31_{\pm 1.75}$     & $71.49_{\pm 0.75}$     \\
    $+$ FA    & OOM                        & $69.90_{\pm 2.33}$     & $80.25_{\pm 5.31}$     & $47.83_{\pm 2.53}$     & $72.90_{\pm 1.42}$     & $72.74_{\pm 0.79}$     \\
    $+$ DIGL  & $84.58_{\pm 1.27}$         & $52.65_{\pm 2.15}$     & $78.50_{\pm 4.19}$     & $41.50_{\pm 3.06}$     & $72.32_{\pm 1.44}$     & $57.62_{\pm 1.01}$     \\
    $+$ SDRF  & $84.55_{\pm 1.40}$         & $69.55_{\pm 2.38}$     & $80.50_{\pm 4.18}$     & $37.17_{\pm 2.71}$     & $69.51_{\pm 1.71}$     & $72.96_{\pm 0.42}$     \\
    $+$ FoSR  & $85.75_{\pm 1.10}$         & $69.25_{\pm 1.81}$     & $80.50_{\pm 4.74}$     & $28.08_{\pm 2.30}$     & $71.52_{\pm 1.77}$     & $71.72_{\pm 0.89}$     \\
    $+$ BORF  & OOT                        & $70.70_{\pm 0.02}$     & $79.25_{\pm 0.04}$     & $34.17_{\pm 0.03}$     & $70.62_{\pm 0.02}$     & OOT                      \\
    $+$ GTR   & $85.47_{\pm 0.83}$         & $69.55_{\pm 1.47}$     & $79.00_{\pm 3.85}$     & $31.75_{\pm 2.47}$     & $72.05_{\pm 1.51}$     & $71.85_{\pm 0.71}$ \\
    $+$ PANDA & \underline{$90.33_{\pm 0.87}$} & $68.35_{\pm 2.35}$ & $83.25_{\pm 3.26}$ & $42.17_{\pm 2.29}$ & $72.32_{\pm 1.79}$     & $73.32_{\pm 0.81}$     \\
    $+$ EGP   & $77.88_{\pm 1.56}$         & $68.25_{\pm 1.12}$     & $81.50_{\pm 4.70}$     & $40.67_{\pm 3.10}$     & $70.85_{\pm 1.57}$     & $72.33_{\pm 0.95}$     \\
    $+$ CGP   & $78.23_{\pm 1.27}$         & \underline{$71.65_{\pm 1.53}$} & \underline{$85.25_{\pm 3.20}$} & $\mathbf{50.08_{\pm 2.24}}$ & \underline{$73.08_{\pm 1.40}$} & \underline{$73.35_{\pm 0.79}$} \\
    \midrule
    $+$ RNRP  & $\mathbf{92.42_{\pm 0.88}}$ & $\mathbf{72.34_{\pm 1.49}}$ & $\mathbf{89.50_{\pm 2.95}}$ & \underline{$48.95_{\pm 2.99}$} & $\mathbf{75.34_{\pm 1.77}}$ & $\mathbf{75.17_{\pm 0.76}}$ \\
    \bottomrule
  \end{tabular}%
  }
  \caption{Comparison of Ramanujan with Non-negative Resistance-curvature Propagation (RNRP) against graph rewiring techniques using \emph{Graph Convolutional Network} (GCN) and \emph{Graph Isomorphism Network} (GIN) backbones on the TU Dataset benchmark. Performance is measured by accuracy. Best scores are shown in \textbf{bold} and second-best scores are \underline{underlined}. OOM: out of memory; OOT: out of time.}
  \label{res_tab}
\end{table}

\begin{table}[ht]
  \centering
  \label{tab:rewiring_methods}
  \begin{tabular}{lcc}
    \toprule
   \textbf{Method} & \textbf{Peptides-func} (AP $\uparrow$) & \textbf{Peptides-struct} (MAE $\downarrow$) \\
    \midrule
    GCN               & $59.30\pm0.23$ & $34.96\pm0.13$ \\
    + FoSR            & $59.47\pm0.35$ & $34.73\pm0.07$ \\
    + GTR             & $50.75\pm0.29$ & $36.18\pm0.10$ \\
    + SDRF            & $59.47\pm0.13$ & $34.78\pm0.13$ \\
    + BORF            & $59.94\pm0.37$ & $35.14\pm0.09$ \\
    + PANDA           & $60.28\pm0.31$ & $32.72\pm0.01$ \\
    + DR           & $61.14\pm0.26$ & $31.87\pm0.11$ \\
    + LASER           & \underline{$64.40\pm0.10$} & \underline{$30.43\pm0.19$} \\
    \hline
    + RNRP              & $\mathbf{66.26 \pm0.27}$ & $\mathbf{28.79\pm0.14}$ \\
    \bottomrule
  \end{tabular}
    \caption{Comparison of RNRP with state of the art rewiring methods on LRGB, performance metrics are Average Precision for Peptides-func and Mean Absolute Error for Peptides-struct. Best scores are shown in \textbf{bold} and second‑best scores are \underline{underlined}.}
    \label{res_tab2}
\end{table}

Table \ref{res_tab} demonstrates that Ramanujan with Nonnegative Resistance curvature Propagation (RNRP) consistently outperforms state‑of‑the‑art rewiring techniques. On average, integrating RNRP with either a GCN \cite{gcn} or GIN \cite{xu2018powerful} backbone yields an accuracy improvement of nearly 12\% compared to the backbone GNN. In particular, when using GIN, RNRP achieves the best results on every dataset except Enzymes, where it still ranks second. 
Although expander graphs \cite{deac2022expander,wilson2024cayley} are the second most effective baselines, our Ramanujan rewiring approach mitigates over-squashing by guaranteeing nonnegative resistance curvature, all while faithfully preserving the graph’s local structure.
\newline 

Table \ref{res_tab2} presents the performance of RNRP on long‑range benchmarks. For the peptides benchmarks, our method consistently improves upon the GCN backbone by an average of 14\% across both tasks and outperforms recent rewiring techniques, thereby demonstrating its superior ability to capture long‑range molecular dependencies.



\subsection{Time Comparison}

\label{Time}

Both predictive accuracy and computational preprocessing cost are key criteria for evaluating graph rewiring methods. In fact, curvature‑based approaches such as those in \cite{UNDERSTANDING_bottlenecks} and \cite{BORF} exhibit quadratic and cubic time complexity in the number of edges, respectively. Graph‑rewiring methods based on effective resistance, such as GTR \cite{resit}, require cubic computational cost in the number of nodes. This high complexity can make preprocessing excessively time‑consuming for large graphs.

\begin{table*}[ht]
  \centering
  \begin{tabular}{lrrrrrr}
    \toprule
    Model   & \textbf{REDDIT-B} & \textbf{IMDB-B} &\textbf{MUTAG}      & \textbf{ENZYMES}    & \textbf{PROTEINS}  & \textbf{  COLLAB}    \\
    \midrule
    DIGL    & 40.3837       & 0.411771    & 0.0342833  & 0.243485   & 0.491458   & 56.3175   \\
    SDRF    & 359.128       & 5.13257     & 0.669701   & 1.71482    & 3.02873    & 619.125   \\
    FoSR    & 74.8568       & 4.54634     & 4.71567    & 4.56855    & 5.04358    & 9.79994   \\
    BORF    & OOT           & 465.408     & 53.7069    & 179.573    & 351.173    & OOT       \\
    GTR     & 118.549       & 3.39839     & 1.54127    & 2.87399    & 6.49714    & 92.6125   \\
    PANDA   & 6.13925       & 0.789759    & 0.246243   & 0.278594   & 0.248043   & 230.850   \\
    EGP     & 0.245215      & 0.0185697   & 0.00446963 & 0.0163198  & 0.0393348  & 0.129567  \\
    CGP     & 0.226065      & 0.0211341   & 0.00438905 & 0.0166841  & 0.0348585  & 0.131188  \\
    \hline
       \textbf{RNRP}    & 1.730475      & 0.1450883  & 0.11660302 & 0.1385242 & 0.1793965  & 1.548660  \\
    \bottomrule
  \end{tabular}
\caption{Comparison of the preprocessing time to construct each graph rewiring method (in seconds per graph). OOT indicates out-of-time.}  
  \label{tab:preprocessing_runtime}
\end{table*}

By contrast, RNRP preprocessing runs in $\mathcal{O}(N^2)$ time, making it over 65× faster than GTR \cite{resit} and over 200× faster than SDRF \cite{UNDERSTANDING_bottlenecks} on the Reddit‑Binary dataset (see Table~\ref{tab:preprocessing_runtime} for time comparison). In practice, this scalability makes RNRP well suited to large, real-world graphs, where more computationally expensive rewiring methods become impractical at scale. Although expander-graph constructions (EGP and CGP) \cite{deac2022expander,wilson2024cayley} are faster to build, RNRP remains the second-fastest method, offering a reasonable balance between preprocessing time and predictive performance. Preprocessing times for all baselines are taken directly from \cite{wilson2024cayley}, and we follow the same evaluation protocol to ensure a fair comparison.

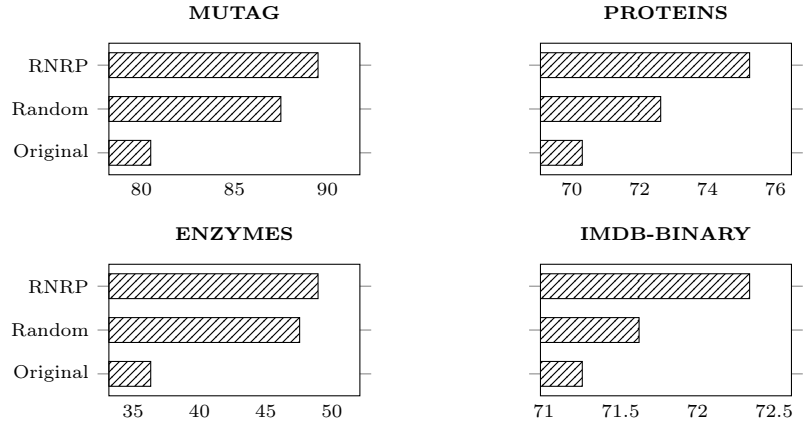
\begin{figure}[t]
\centering

\begin{minipage}[b]{0.49\linewidth}
  \centering
  \begin{tikzpicture}
    \begin{axis}[
      width=0.82\linewidth,
      height=0.62\linewidth,
      bar width=0.055\linewidth,
      xbar,
      y=0.58cm,
      enlargelimits=0.25,
      yticklabels={
        {\scriptsize Original},
        {\scriptsize Random},
        {\scriptsize RNRP}
      },
      ytick=data,
      title={\textbf{\scriptsize MUTAG}},
      tick label style={font=\scriptsize},
      label style={font=\scriptsize},
      title style={font=\scriptsize},
      xtick style={draw=none}
    ]
      \addplot[pattern=north east lines] coordinates {
        (80.50,0) (87.50,1) (89.50,2)
      };
    \end{axis}
  \end{tikzpicture}
\end{minipage}\hfill
\begin{minipage}[b]{0.49\linewidth}
  \centering
  \begin{tikzpicture}
    \begin{axis}[
      width=0.82\linewidth,
      height=0.65\linewidth,
      bar width=0.055\linewidth,
      xbar,
      y=0.58cm,
      enlargelimits=0.25,
      yticklabels={},
      ytick=data,
      title={\textbf{\scriptsize PROTEINS}},
      tick label style={font=\scriptsize},
      label style={font=\scriptsize},
      title style={font=\scriptsize},
      xtick style={draw=none}
    ]
      \addplot[pattern=north east lines] coordinates {
        (70.31,0) (72.62,1) (75.24,2)
      };
    \end{axis}
  \end{tikzpicture}
\end{minipage}

\vspace{0.5em}

\begin{minipage}[b]{0.49\linewidth}
  \centering
  \begin{tikzpicture}
    \begin{axis}[
      width=0.82\linewidth,
      height=0.65\linewidth,
      bar width=0.055\linewidth,
      xbar,
      y=0.58cm,
      enlargelimits=0.25,
      yticklabels={
        {\scriptsize Original},
        {\scriptsize Random},
        {\scriptsize RNRP}
      },
      ytick=data,
      title={\textbf{\scriptsize ENZYMES}},
      tick label style={font=\scriptsize},
      label style={font=\scriptsize},
      title style={font=\scriptsize},
      xtick style={draw=none}
    ]
      \addplot[pattern=north east lines] coordinates {
        (36.32,0) (47.56,1) (48.95,2)
      };
    \end{axis}
  \end{tikzpicture}
\end{minipage}\hfill
\begin{minipage}[b]{0.49\linewidth}
  \centering
  \begin{tikzpicture}
    \begin{axis}[
      width=0.82\linewidth,
      height=0.65\linewidth,
      bar width=0.055\linewidth,
      xbar,
      y=0.58cm,
      enlargelimits=0.25,
      yticklabels={},
      ytick=data,
      title={\textbf{\scriptsize IMDB-BINARY}},
      tick label style={font=\scriptsize},
      label style={font=\scriptsize},
      title style={font=\scriptsize},
      xtick style={draw=none}
    ]
      \addplot[pattern=north east lines] coordinates {
        (71.25,0) (71.62,1) (72.34,2)
      };
    \end{axis}
  \end{tikzpicture}
\end{minipage}

\caption{Ablation study of classification accuracy on the original graph, a randomly generated Ramanujan graph, and our \textbf{RNRP} method.}
\label{fig:three_methods_comparison}
\end{figure}

\subsection{Ablation Studies}
\label{ablation}

In this part we conduct ablation experiments to empirically evaluate (i)  the impact of preserving original node proximity during Ramanujan-based rewiring and (ii) assess robustness to oversmoothing.

\paragraph{\textbf{Effect of preserving proximity.}}
We compare on MUTAG, PROTEINS, ENZYMES, and 
IMDB-BINARY datasets: (a) the original input graph, (b) a $d$-regular random
Ramanujan $RN$ graph (same degree, no proximity guidance), and (c) our proximity-aware RNRP. All models use the same training protocol as described in Section ~\ref{XXXP} with a GIN backbone. As reported in Figure~\ref{fig:three_methods_comparison}, the random Ramanujan RN graph already improves over the original topology, indicating that expansion and uniform degree help long-range propagation. Yet, RNRP consistently attains the highest accuracy across datasets, showing that combining Ramanujan RN expansion with preservation of original structure proximity yields the strongest gains, naive expander insertion, while beneficial, is insufficient to match RNRP.

\paragraph{\textbf{Robustness to oversmoothing.}}

We quantify oversmoothing via the node wise Dirichlet energy defined in equation \ref{diri} and track its evolution during training for a backbone GCN on the original graphs and the same GCN trained on the RNRP graphs. On ENZYMES, MUTAG, and PROTEINS, Figure \ref{oversm} shows that the Dirichlet energy remains systematically higher with RNRP than the original graph, indicating slower contraction of intra-neighborhood variance. This is consistent with our rewiring approach : the $d$-regular Ramanujan structure under the $RN$ constraint supplies many low-resistance alternative routes (improving global reach) while equalizing per-node aggregation (reducing hub-driven averaging), thus mitigating oversmoothing without sacrificing long-range propagation.

\begin{figure}[t]
  \centering
  \begin{minipage}[b]{0.33\textwidth}
    \centering
    \includegraphics[width=\linewidth]{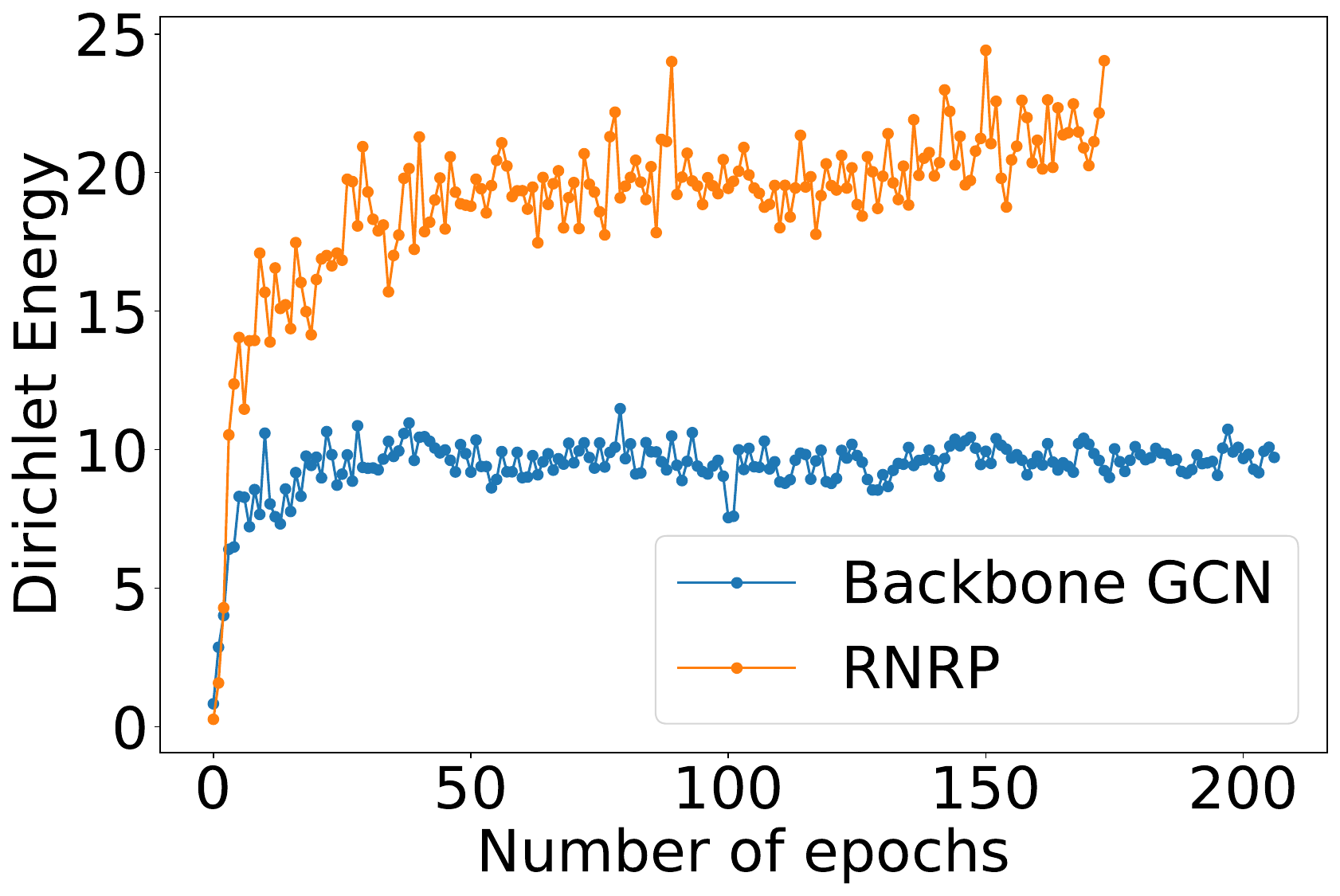}\\
    \small (a) ENZYMES
  \end{minipage}\hfill
  \begin{minipage}[b]{0.33\textwidth}
    \centering
    \includegraphics[width=\linewidth]{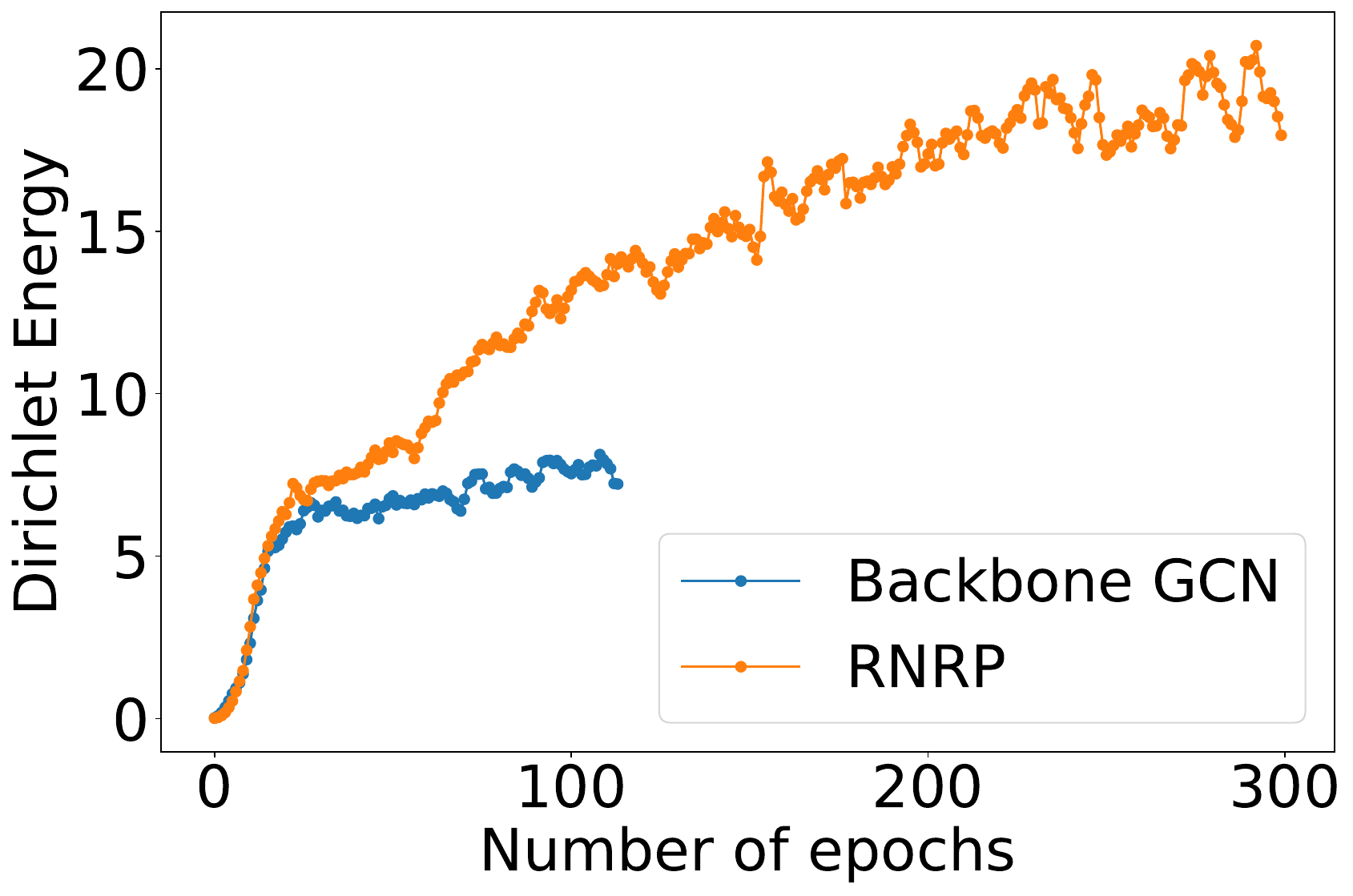}\\
    \small (b) MUTAG
  \end{minipage}\hfill
  \begin{minipage}[b]{0.33\textwidth}
    \centering
    \includegraphics[width=\linewidth]{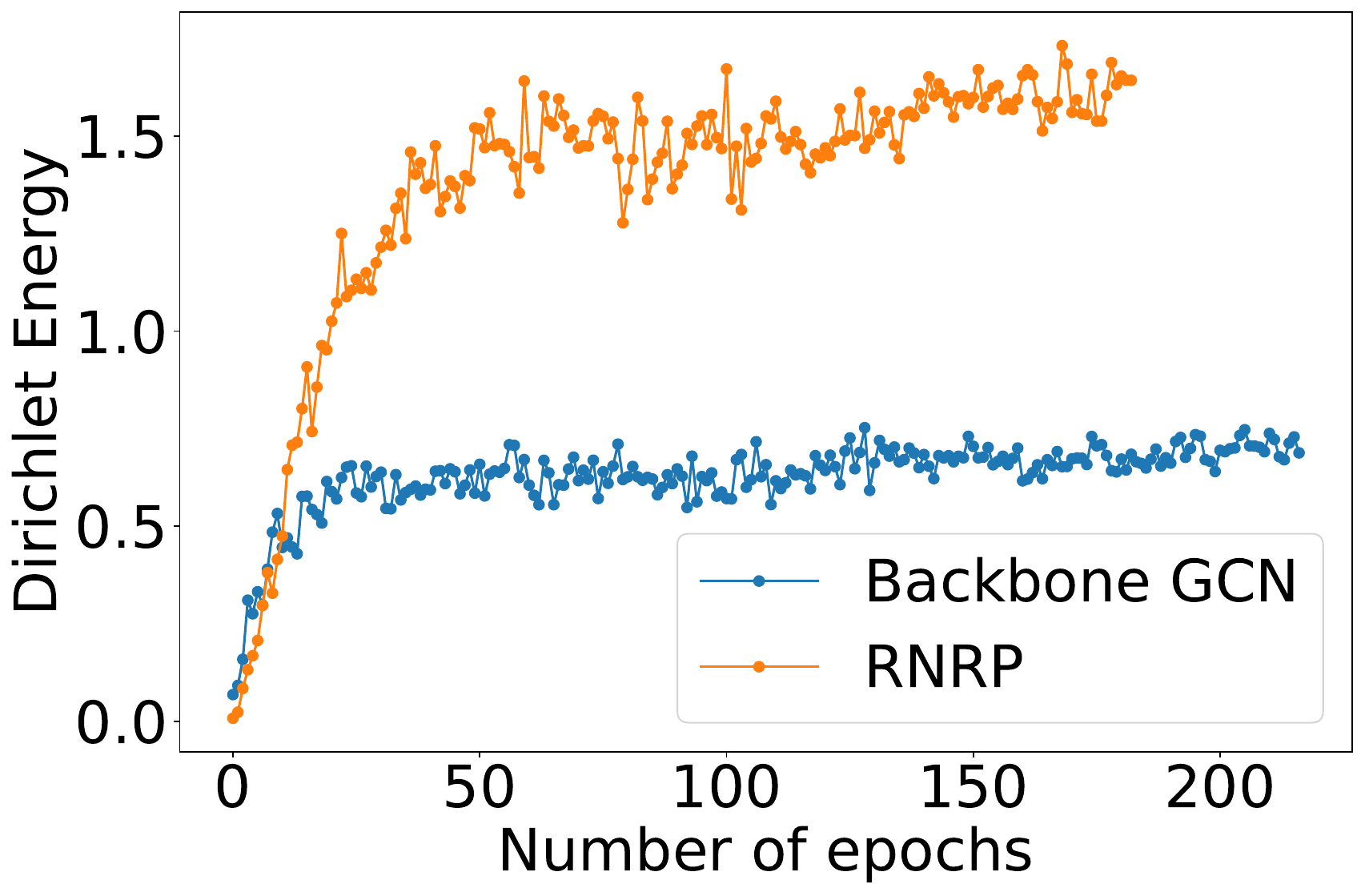}\\
    \small (c) PROTEINS
  \end{minipage}

  \caption{Comparison of Dirichlet energy over training epochs for a GCN on the original graph versus the \textbf{RNRP}-rewired graph. Higher values reflect weaker homogenization; \textbf{RNRP} consistently reaches larger energy, reduced over-smoothing relative to the original structure.}
  \label{oversm}
\end{figure}

\section{Conclusion}

In this paper, we propose Ramanujan with Nonnegative Resistance curvature Propagation (RNRP), a graph‑rewiring method that integrates the strong spectral properties of Ramanujan graphs with the Non-negative resistance curvature class graph. We provide theoretical guarantees showing that this new graph alleviates information over-squashing and remains robust to oversmoothing while preserving essential topological properties such as low graph diameter and low effective resistance. Empirical results demonstrate that RNRP achieves state‑of‑the‑art results on tasks requiring long‑range dependencies, outperforming established graph‑rewiring methods while maintaining a competitive preprocessing cost.

\section{Declaration of Generative AI Use}

This paper includes text that was revised and refined using generative AI tools (ChatGPT) to improve clarity and correct potential grammatical errors.


%
%
%
\bibliographystyle{splncs04}
\bibliography{ref}

\appendix

\section{Appendix: Proofs of Sections 4 and 5}

\setcounter{theorem}{0}
\setcounter{corollary}{0}

\subsection{Proof of Theorem 1: Commute Time Bound via Resistance Curvature}
\label{CTCURV}

\begin{theorem}[Commute time bound via resistance curvature]
Let \(G=(V,E)\) be a finite, connected, undirected, and unweighted graph.
For any \(u,v\in V\), let
\[
u=x_0\sim x_1\sim \cdots \sim x_k=v
\]
be a shortest path, so that \(k=d_G(u,v)\). Define
\[
p_\star:=\min_{0\le i\le k} p(x_i),
\qquad\text{where}\qquad
p(x)=1-\frac12\sum_{y\in N(x)}\Reff(x,y).
\]
Then
\[
\CT(u,v)\le 4\,|E|\,k\,\bigl(1-p_\star\bigr).
\]
\end{theorem}

\begin{proof}
For the simple random walk on \(G\), the commute time satisfies
\begin{equation}\label{eq:ct_reff_appendix}
\CT(u,v)=2|E|\,\Reff(u,v),
\end{equation}
see \cite{chandra1989electrical}. Moreover, \(\Reff\) defines a metric on \(V\), 
and therefore satisfies the triangle inequality. Applying this inequality 
along the shortest path
\[
u=x_0\sim x_1\sim \cdots \sim x_k=v
\]
gives
\begin{equation}\label{eq:triangle_reff_appendix}
\Reff(u,v)\le \sum_{i=0}^{k-1}\Reff(x_i,x_{i+1}).
\end{equation}

We now bound each term \(\Reff(x_i,x_{i+1})\) through the resistance curvature. 
By definition,
\[
1-p(x)=\frac12\sum_{y\in N(x)}\Reff(x,y).
\]
Since all terms in the sum are nonnegative, for every neighbor \(y\in N(x)\) we have
\begin{equation}\label{eq:neighbor_bound_appendix}
\Reff(x,y)\le \sum_{z\in N(x)}\Reff(x,z)=2\bigl(1-p(x)\bigr).
\end{equation}
In particular, for each edge \((x_i,x_{i+1})\) on the path,
\[
\Reff(x_i,x_{i+1})\le 2\bigl(1-p(x_i)\bigr).
\]

By definition of
\[
p_\star=\min_{0\le i\le k}p(x_i),
\]
we have \(p(x_i)\ge p_\star\) for all \(i\), and hence
\[
1-p(x_i)\le 1-p_\star.
\]
Therefore,
\[
\Reff(x_i,x_{i+1})\le 2(1-p_\star),
\qquad\text{for all } i=0,\dots,k-1.
\]
Combining this estimate with \eqref{eq:triangle_reff_appendix}, we obtain
\[
\Reff(u,v)
\le \sum_{i=0}^{k-1}\Reff(x_i,x_{i+1})
\le \sum_{i=0}^{k-1}2(1-p_\star)
=2k(1-p_\star).
\]
Substituting this into \eqref{eq:ct_reff_appendix} yields
\[
\CT(u,v)
=2|E|\,\Reff(u,v)
\le 2|E|\cdot 2k(1-p_\star)
=4|E|k(1-p_\star),
\]
which concludes the proof.
\end{proof}

\subsection{Proof of Theorem 2: Ramanujan Graphs Have Non-Negative Resistance Curvature}
\label{RNGRAPH}

\begin{theorem}
\label{thm:valued_appendix}
Let \(G\) be a Ramanujan graph on \(N\) vertices of degree \(d\), 
and let \(C>0\) be the absolute constant from \cite{glock2024hamilton}. 
If
\[
d \ge 4C^2(\log N)^{2/3},
\]
then \(G\) has non-negative resistance curvature.
\end{theorem}

\begin{proof}
By the Ramanujan property,
\[
\mu(G)\le 2\sqrt{d-1}<2\sqrt d,
\]
and therefore
\[
\frac{d}{\mu(G)}>\frac{\sqrt d}{2}.
\]
By \cite{glock2024hamilton}, any \((N,d,\mu)\)-graph satisfying
\[
\frac{d}{\mu(G)}\ge C(\log N)^{1/3}
\]
is Hamiltonian. Moreover, Hamiltonian graphs have non-negative resistance 
curvature by \cite{devriendt2024graph,devriendt2024graphs}. Hence, it is sufficient to require
\[
\frac{\sqrt d}{2}\ge C(\log N)^{1/3},
\]
which is equivalent to
\[
d\ge 4C^2(\log N)^{2/3}.
\]
This proves the claim.
\end{proof}

\subsection{Proof of Corollary 1: Low Diameter}
\label{diamgraph}

\begin{corollary}
\label{cor:diam_appendix}
Let \(G\) be a finite, connected Ramanujan graph on \(N\) vertices of degree \(d\) 
with non-negative resistance curvature. Then
\[
\operatorname{diam}(G)\le 1+\frac{3\,\log N}{\log(\log N)}.
\]
\end{corollary}

\begin{proof}
Let \(G\) be a connected \(d\)-regular non-bipartite graph on \(N\) vertices, 
and let \(\mu\) denote its second-largest adjacency eigenvalue. 
By Chung's adjacency-power bound \cite{chung1997spectral},
\[
\operatorname{diam}(G)\le 1+\frac{\log(N-1)}{\log(d/\mu)}.
\]
If \(G\) is Ramanujan, then \(\mu\le 2\sqrt{d-1}\), and thus
\[
\operatorname{diam}(G)\le 1+\frac{\log(N-1)}{\log\!\left(\dfrac{d}{2\sqrt{d-1}}\right)}.
\]
By Remark~\ref{remark1_}, we have \(d \ge 4(\log N)^{2/3}\). Hence,
\[
\frac{d}{2\sqrt{d-1}}
\ge
\frac{4(\log N)^{2/3}}{2\sqrt{\,4(\log N)^{2/3}-1\,}}
\ge
\frac{4(\log N)^{2/3}}{2\cdot 2(\log N)^{1/3}}
=
(\log N)^{1/3}.
\]
Taking logarithms gives
\[
\log\!\left(\frac{d}{2\sqrt{d-1}}\right)\ge \frac13\log(\log N).
\]
Therefore,
\[
\operatorname{diam}(G)
\le
1+\frac{\log(N-1)}{\frac13\log(\log N)}
=
1+\frac{3\,\log(N-1)}{\log(\log N)}.
\]
Since \(\log(N-1) \le \log N\), we conclude
\[
\operatorname{diam}(G)\le 1+\frac{3\,\log N}{\log(\log N)}.
\]
\end{proof}

\subsection{Proof of Corollary 2: Maximum Effective Resistance}
\label{RESGRAPH}

\begin{corollary}
\label{cor:rmax_appendix}
Let \(G\) be a Ramanujan \(RN\) graph of degree \(d\) on \(N\) vertices. 
Then the maximum effective resistance \(R_{\max}\) satisfies
\[
R_{\max}\le \frac{2}{d-2\sqrt{d-1}}.
\]
\end{corollary}

\begin{proof}
Combining the general bounds from \cite{chandra1989electrical,resit}, we have
\[
\frac{1}{N\lambda_2}\le R_{\max}\le \frac{2}{\lambda_2},
\]
where \(\lambda_2\) denotes the spectral gap. Since \(G\) is Ramanujan,
\[
\lambda_2=d-\mu_2\ge d-2\sqrt{d-1}.
\]
Therefore,
\[
R_{\max}\le \frac{2}{\lambda_2}\le \frac{2}{d-2\sqrt{d-1}}.
\]
This concludes the proof.
\end{proof}
\section{Description of graph classification datasets}

In this section we report the main structural statistics of the graph classification datasets used in our experiments, as they help contextualize the behavior of the rewiring procedure across benchmarks with different graph sizes and densities.

\begin{table}[h!]
\centering
\footnotesize
\setlength{\tabcolsep}{6pt}
\begin{tabular}{lrrrrr}
\toprule
\textbf{Dataset} & \textbf{\#Graphs} & \textbf{\#Nodes} & \textbf{\#Edges} & $\mathbf{\overline{d}}$ & $\mathbf{\mathrm{Std}(d)}$   \\
\midrule
MUTAG           & 188    & 17.9 ($\pm$5)     & 19.8 ($\pm$6)     & 2.2 & 0.7    \\
PROTEINS        & 1113   & 39.1 ($\pm$46)    & 72.8 ($\pm$85)    & 3.7 & 0.9    \\
ENZYMES         & 600    & 32.6 ($\pm$15)    & 62.1 ($\pm$25)    & 3.9 & 1.0    \\
REDDIT          & 2000   & 429.6 ($\pm$554)  & 497.8 ($\pm$623)  & 2.3 & 8.9   \\
COLLAB          &   5000  & 74.5 (±62)  & 2457.2 (±6439)  & 37.4  & 11.8  \\
IMDB            & 1000   & 19.8 ($\pm$10)    & 96.5 ($\pm$105)   & 8.9 & 2.8   \\
\bottomrule
\end{tabular}
\caption{Graph-level datasets: average structural statistics across graphs (mean $\pm$ standard deviation). $\overline{d}$, $\mathrm{Std}(d)$, denote mean, standard deviation.}
\label{tab:dataset_stats_graph}
\end{table}

\begin{itemize}


\item \textbf{Molecular datasets.} MUTAG contains small molecular graphs: nodes represent atoms, edges represent chemical bonds, and node labels encode atom types. 
PROTEINS and ENZYMES represent protein structures: nodes are secondary-structure elements with physicochemical annotations, and edges connect elements that are adjacent in the amino-acid sequence or spatially close in 3D. 
In PROTEINS, graph labels are binary. 
In ENZYMES, graphs are classified into the six top-level EC classes.

\item \textbf{Social datasets.} 
IMDB-BINARY is constructed from actor collaboration graphs: nodes represent actors/actresses, and an edge indicates that two actors co-appeared in a movie. Each graph is assigned one of two movie-genre labels. 
COLLAB is built from researcher ego-networks, where nodes are researchers and edges represent collaboration links; graph labels correspond to research fields. 
REDDIT-BINARY is derived from discussion-thread interaction graphs, in which nodes are users and edges denote reply interactions; each graph label indicates the discussion community type.

\end{itemize}

%




\end{document}